\documentclass[11pt,a4paper]{article}
\usepackage[hyperref]{acl2021}
\usepackage{times}
\usepackage{latexsym}

\usepackage{enumitem}
\usepackage{amsthm}
\usepackage{amsfonts}
\newtheorem{definition}{Definition}
\newtheorem{theorem}{Theorem}
\newtheorem{corollary}{Corollary}
\newtheorem{lemma}{Lemma}
\usepackage{graphicx}
\usepackage{caption}
\usepackage{subcaption}
\usepackage{multirow}
\usepackage{booktabs}
\usepackage{amsfonts}
\usepackage{amsmath} 
\usepackage{bm}  

\usepackage{microtype}
\usepackage{amsmath,amssymb}
\newcommand{\seq}[1]{\boldsymbol{#1}}
\newcommand{\x}{\boldsymbol{x}}
\newcommand{\y}{\boldsymbol{y}}
\aclfinalcopy 

\usepackage{color} 
\definecolor{mypink2}{RGB}{219, 48, 122}
\definecolor{jrcolor}{RGB}{100, 150, 225}
\definecolor{jrcomment}{RGB}{70, 200, 150}

\title{Data Augmentation for Text Generation Without Any Augmented Data}

\author{Wei Bi\footnotemark[1]~~~Huayang Li\thanks{~~Equal contribution.}~~~Jiacheng Huang \\
        Tencent AI Lab, Shenzhen, China\\
        \tt \{victoriabi,alanili,eziohuang\}@tencent.com 
}

\date{}

\begin{document}
\maketitle
\begin{abstract}
Data augmentation is an effective way to improve the performance of many neural text generation models. However, current data augmentation methods need to define or choose proper data mapping functions that map the original samples into the augmented samples. In this work, we derive an objective to formulate the problem of data augmentation on text generation tasks without any use of augmented data constructed by specific mapping functions. Our proposed objective can be efficiently optimized and applied to popular loss functions on text generation tasks with a convergence rate guarantee. Experiments on five datasets of two text generation tasks show that our approach can approximate or even surpass popular data augmentation methods.
\end{abstract}

\section{Introduction}

End-to-end neural models are generally trained in a data-driven paradigm.
Many researchers have proposed powerful network structures to fit training data well. 
It has also become ubiquitous to increase the training data amount to improve model performance. Data augmentation is an effective technique to create additional samples in both vision and
text classification tasks~\cite{perez2017effectiveness,shorten2019survey,wei2019eda}, which perturb samples without changing their labels.
For text generation tasks, there can be more types of data perturbation to construct augmented samples, including corrupting the input text~\cite{xie2019data}, the output text~\cite{norouzi2016reward,kurata2016labeled}, or both~\cite{zhang2020dialogue}. As such,
classification tasks can be regarded as special cases of generation tasks in terms of incorporating data augmentation techniques, and this work mainly discusses text generation tasks.

The focus of previous work on text data augmentation has been to design proper augmentation techniques to create augmented samples.
Some augmentation methods have been proposed for general text tasks.
For example, different general replacement operations have been explored to edit words in a text sample, 
ranging from simple look-up tables~\cite{zhang2015character} to pretrained masked language models~\cite{kobayashi2018contextual,wu2019conditional}.
\citet{sennrich-etal-2016-improving} propose to augment text sequences by back-translation.
For some generation tasks such as dialogue generation, general augmentation methods may not yield stable improvements and it requires to carefully incorporate the task property to design useful augmented samples~\cite{zhang2020dialogue}.
All these methods need to explicitly construct augmented samples, and the data mapping functions from the original samples to the augmented samples are mostly defined apriori.
This motivates us to raise a question, whether we can skip the step to define or choose proper augmented data mapping functions to accomplish effective data augmentation. 

To answer this question, we aim to formulate the problem of data augmentation for general text generation models without any use of augmented data mapping functions.
We start from a conventional data augmentation objective, which is a weighted combination of loss functions associated with the original and augmented samples.
We show that the loss parts of the augmented samples can be re-parameterized by variables not dependent on the augmented data mapping functions, if a simple Euclidean loss function between the sentence representations is applied. 
Based on this observation, we propose to directly define a distribution on the re-parameterized variables. 
Then we optimize the expectation of the augmented loss parts over this distribution to approximate the original augmented loss parts computed with various augmented data mapping functions. 
We make different assumptions on the variable distributions and 
find that our proposed objective can be computed and optimized efficiently by simple gradient weighting.
If stochastic gradient descent (SGD) is used, our objective is guaranteed with the convergence rate $O(1/\sqrt{T})$.   
Our objective can be coupled with popular loss functions on text generation tasks, including the word mover's distance~\cite{kusner2015word} and the cross-entropy loss. 

Our approach, which utilizes the proposed objective and optimizes it by SGD,
has two advantages. 
First, it provides a unified formulation of various data perturbation types in general text generation models, which sheds a light on understanding the working mechanism of data augmentation.
Second, the optimization of our approach is simple and efficient. Without introducing any new sample during training, we can avoid additional calculation efforts on augmented samples, often with the total size much larger than the original data size.
Hence, our approach maintains high training efficiency.

Extensive experiments are conducted to validate the effectiveness of our approach. We mainly use the LSTM-based network structure~\cite{bahdanau2014neural, luong2015effective} and perform experiments on two text generation tasks - neural machine translation and single-turn conversational response generation. Results on five datasets demonstrate that the proposed approach can approximate or even surpass popular data augmentation methods such as masked language model~\cite{devlin2019bert} and back-translation~\cite{sennrich-etal-2016-improving}.


\section{Related Work}
Data augmentation has shown promising improvements on neural models for different text generation tasks such as language modeling~\cite{xie2019data}, machine translation~\cite{sennrich-etal-2016-improving} and dialogue generation~\cite{niu2019automatically,cai2020}.
Existing text data augmentation methods can be mainly categorized into word-level augmentation and sentence-level augmentation.

Word-level augmentation methods perturb words within the original sentence.
Common operations include word insertion and deletion~\cite{wei2019eda}, synonym replacement~\cite{zhang2015character}, and embedding mix-up~\cite{guo2019augmenting}.
Masked language models can be used by masking some percentages of tokens at random, and predicting the masked
words based on its context~\cite{wu2019conditional,cai2020}.

Sentence-level data augmentation is not limited to edit only a few words in the original sentence, but to generate a complete sentence.
For example, back-translation is originally proposed to translate monolingual target language data into source language to augment training pairs in machine translation~\cite{sennrich-etal-2016-improving}. It is later extended to paraphrase sentences in any text dataset, in which two translation models are applied: one
translation model from the source language to target language and another from the target to the source.
GAN-based and VAE-based models have also achieved impressive results to create entire sentences to augment the training data~\cite{hu2017toward,cheng2019robust}.
For dialogue generation, retrieved sentences can be good supplement of the original corpus~\cite{zhang2020dialogue}.

Both word-level and sentence-level augmentation methods need to define their augmented data mapping functions (i.e. operations to edit words or models to generate sentences) apriori. Some works train policies to sample a set of word-level operations~\cite{niu2019automatically}, but the operation candidates are still pre-defined. 
A few works learn to construct augmented samples and optimize the network jointly~\cite{hu2019learning,cai2020}.
Different from previous work, our goal is not to propose or learn novel augmented data mapping functions. Instead, we investigate whether the effectiveness of data augmentation can be achieved while we do not bother to use any specific augmented data mapping function. 

Besides data augmentation, data weighting is another useful way to improve model learning. It assigns a weight to each sample to adapt its importance during training. The sample weights are often carefully defined~\cite{freund1997decision,bengio2009curriculum} or learnt by another network~\cite{jiang2018mentornet,shu2019meta}. Data augmentation is often combined with data weighting together to weight the original and augmented samples. 

\section{Background}
\label{sec:back}
We are given original samples $\mathcal{D}=\{(\x,\y)\}$ with $\x,\y$ both as text sequences. Without loss of generality, a deep generation model is to learn a mapping function $f_{x,y}$ by a deep neural network that outputs $\y$ given $\x$. 
As mentioned in the introduction, text generation tasks mainly have three types of augmented data: 
\begin{itemize}[wide=0\parindent,noitemsep, topsep=0pt]
\item one (or several) perturbed input text $\hat{\x}$ by one (or several) augmented data mapping function $\phi_{\hat{x}}$;
\item one (or several) perturbed output text $\hat{\y}$ by one (or several) augmented data mapping functions $\phi_{\hat{y}}$;
\item one (or several) perturbed paired text $(\hat{\x}, \hat{\y})$ by corresponding augmented data mapping functions.
\end{itemize}
%
Proper augmented data mapping functions are often supposed to generate perturbed sequences or sequence pairs that are close to the original one. They are assumed to be given apriori in optimizing the generation model for now. 

Let $\ell(f_{x,y}(\x), \y)$ denote the loss function to be minimized for each sample.
We first use augmented data in the input domain as an example to present the problem formulation and introduce our approach, then later discuss other types of augmented data.
Data augmentation methods generally apply an augmented loss per sample with its augmented samples:

\begin{equation}
\ell_{aug} = \ell (f_{x,y}(\x), \y) + \sum_{\hat{x}: \phi_{\hat{x}} \in \mathcal{F}} w_{\hat{x}} \ell (f_{x,y}(\hat{\x}), \y) \label{eq:obj}
\end{equation}
where $w_{\hat{x}}$ is the importance weight
associated with each augmented sample, $\phi_{\hat{x}}$ is the augmented data mapping function that constructs $\hat{x}$, and $\mathcal{F}$ is the function space containing all feasible augmented data mapping functions.

\section{Our Approach}
\label{sec:model}
In this section, we aim to formulate the problem of data augmentation for general text generation models without any use of augmented data mapping functions.
We introduce our approach by assuming that the loss function $\ell$ is the most simple Euclidean distance, i.e. 
\begin{eqnarray}
\ell(\seq{u},\seq{v}) = \|\seq{u} - \seq{v}\|_2 \label{eq:l2}
\end{eqnarray}
where $\seq{u}$ and $\seq{v}$ are the sentence representations of two sentences, i.e. the target sequence and the predicted sequence. 
Other conventional loss functions in text generation will be discussed in Section~\ref{sec:loss}.

We first rewrite each loss part of an augmented data point in (\ref{eq:obj}) from a polar coordinate system in Sec~\ref{sec:obj}.
In this way, we can regard the total augmented loss part with multiple augmented data mapping functions as sampling different points in the polar coordinate system. 
This inspires us that we can skip to define any augmented data mapping function, but only design a joint distribution of the perturbation radius and perturbation angle in the polar coordinate system. 
In Sec~\ref{sec:opt}, we show two probability distribution substantiations, and find that our approach can be optimized efficiently by simply re-weighting the gradients. 
In Sec~\ref{sec:types}, we discuss the extension of our approach for other augmented data mapping function types. 

\subsection{Proposed Objective}
\label{sec:obj}

\begin{figure*}
     \centering
     \begin{subfigure}[b]{0.3\textwidth}
         \centering
         \includegraphics[width=\textwidth]{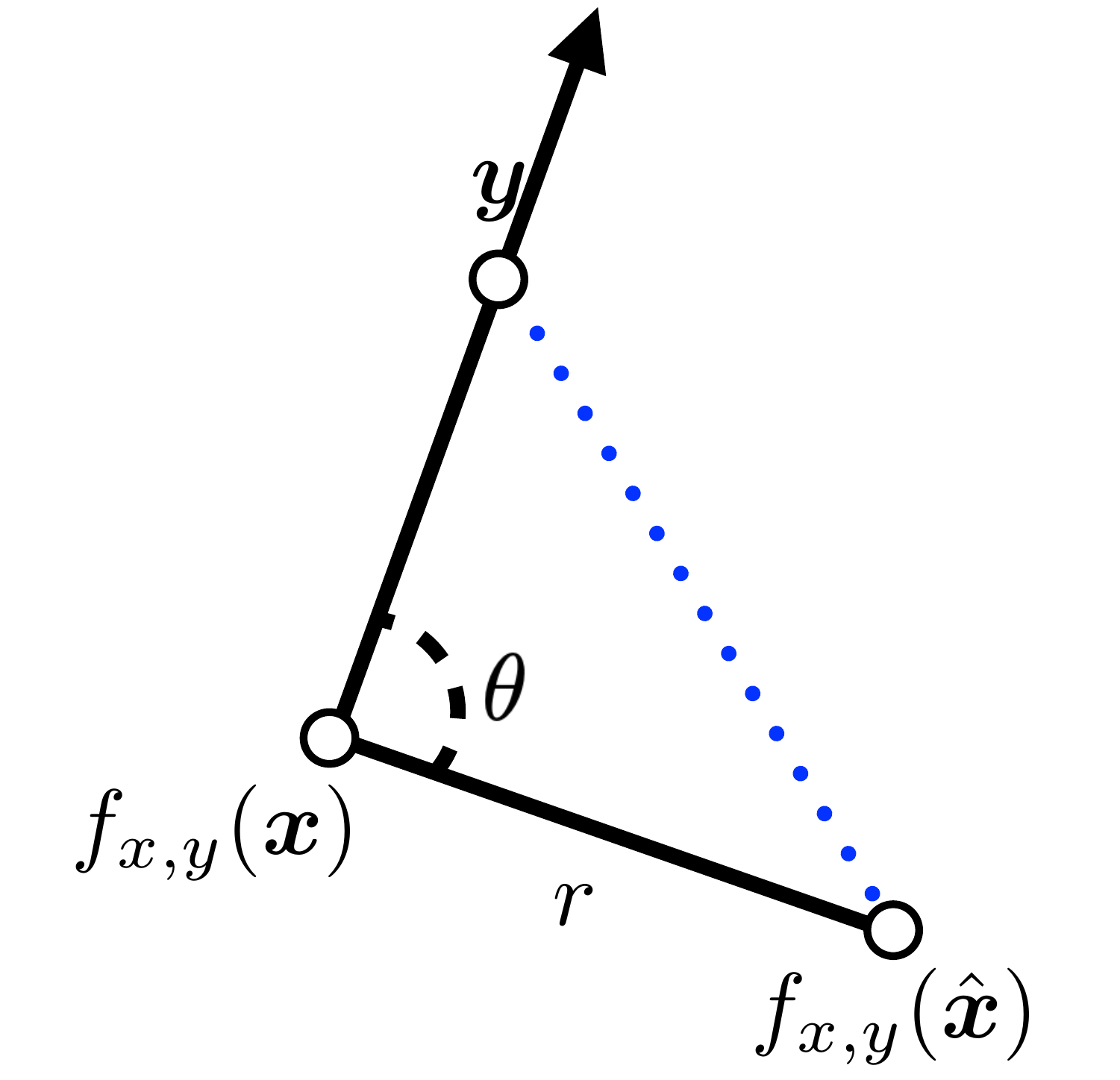}
         \caption{with a perturbed input $\hat{\x}$}
         \label{fig1}
     \end{subfigure}
     \hfill
     \begin{subfigure}[b]{0.3\textwidth}
         \centering
         \includegraphics[width=\textwidth]{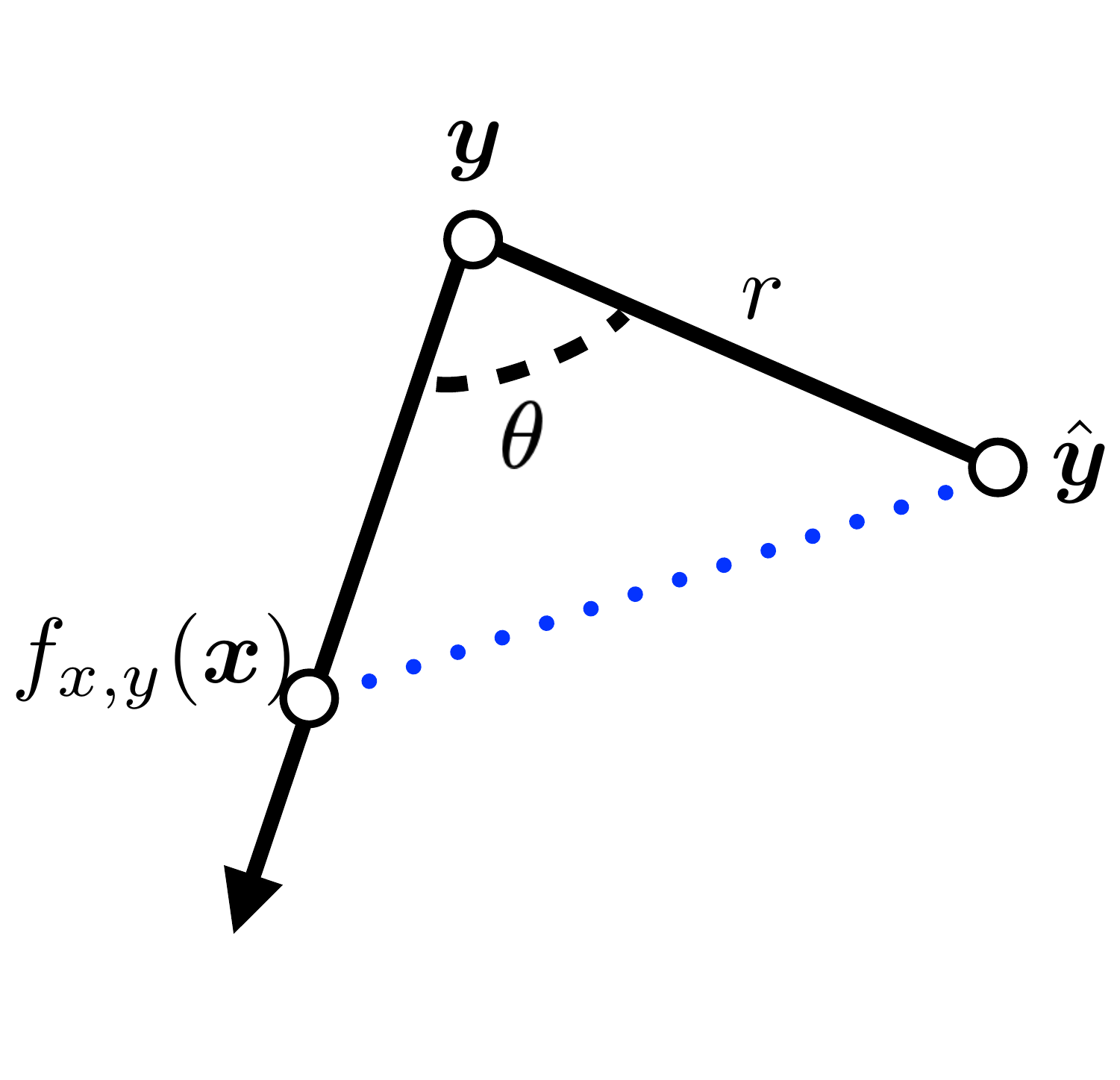}
         \caption{with a perturbed output $\hat{\y}$}
         \label{fig2}
     \end{subfigure}
     \hfill
     \begin{subfigure}[b]{0.3\textwidth}
         \centering
         \includegraphics[width=\textwidth]{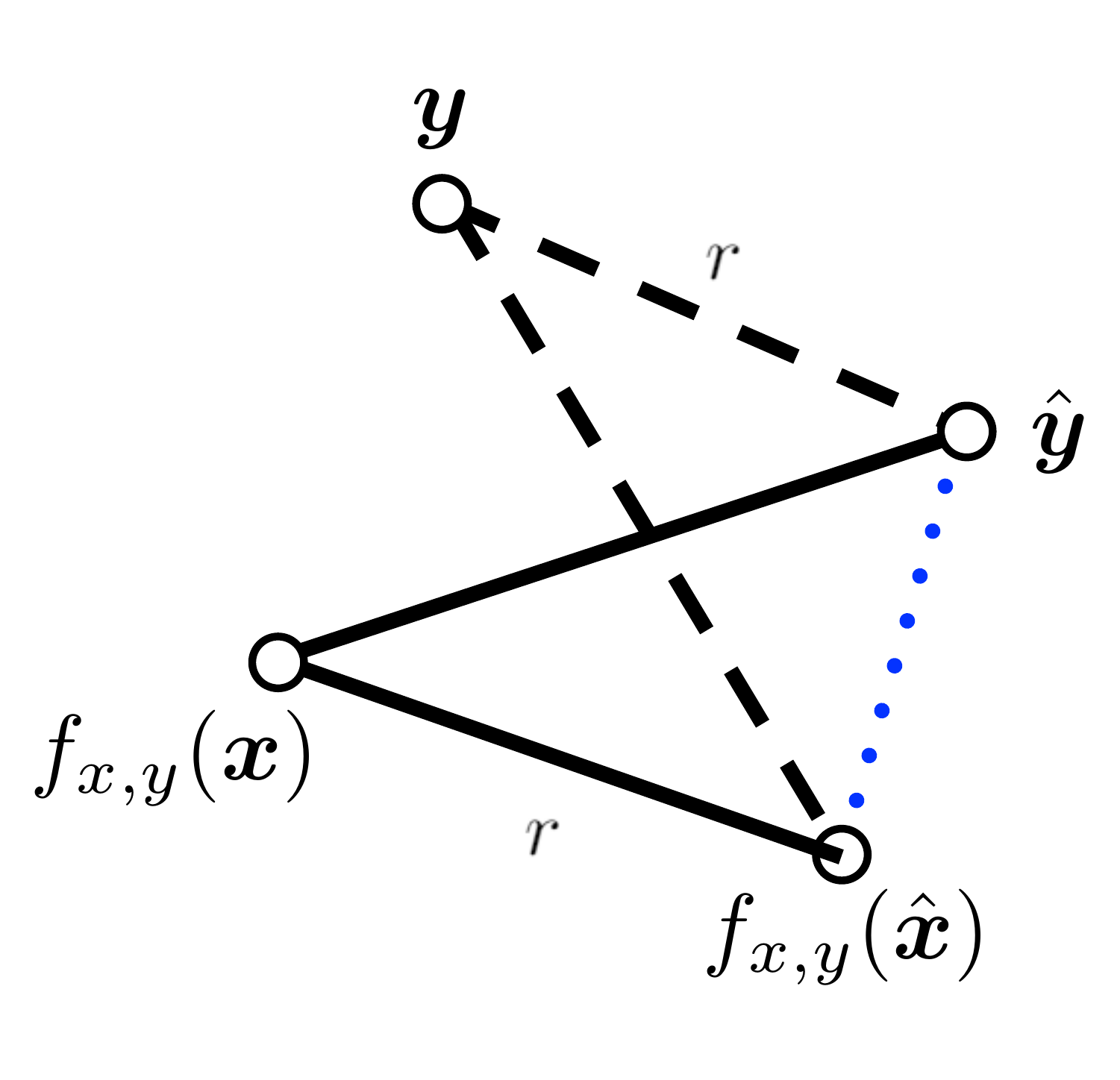}
         \caption{with a perturbed paired text $(\hat{\x},\hat{\y})$}
         \label{fig3}
     \end{subfigure}
        \caption{Illustration of the polar coordinate systems for three kinds of data perturbation. Rays in the figures are the polar axes. Our approach expresses edges in dots by their corresponding polar coordinates.}
    
\end{figure*}


By treating $f_{x,y}(\x), f_{x,y}(\hat{\x})$ and $\y$ as three vertices in the Euclidean space, we can form a triangle (illustrated in Fig.~\ref{fig1}) with the three vertices and the loss between them as edges. 
For a given augmented data mapping function $\phi_{\hat{x}}$ and a sample $(\x, \y)$, 
we can rewrite $\ell (f_{x,y}(\hat{\x}), \y)$ using the polar coordinate system with $f_{x,y}(\x)$ as the pole and $(f_{x,y}(\x), \y)$ as the polar axis:
\begin{eqnarray}
& \lefteqn{\ell^2 (f_{x,y}(\hat{\x}), \y) =} \nonumber\\
&&  \ell^2(f_{x,y}(\x), \y) + \ell^2 (f_{x,y}(\x), f_{x,y}(\hat{\x})) \nonumber\\
&& - 2\ell (f_{x,y}(\x), f_{x,y}(\hat{x}))\ell (f_{x,y}(\x), \y)\cos\theta \nonumber\\ \label{eq:tri}
\end{eqnarray}
where $\theta$ is the radian of $f_{x,y}(\hat{\x})$.
We can observe that, the rewritten augmented sample loss part depends on the original sample loss $\ell(f_{x,y}(\x), \y)$ as well as the radius $r$ and radian $\theta$  of $f_{x,y}(\hat{\x})$. Here $r$ is the data perturbation distance $\ell(f_{x,y}(x), f_{x,y}(\hat{x}))$.
%
Therefore, we can map each augmented data mapping function $\phi_{\hat{x}} \in \mathcal{F}$ into $(r, \theta) \in P$, where 
$P$ is a joint distribution of $(r, \theta)$
~\footnote{
It is worth pointing out that even if 
the three vertices (i.e.,
$f_{x,y}(\hat{\x})$, $\y$, and $f_{x,y}(\x)$ ) lie in high dimensional spaces,
we can always use the distribution of $(r, \theta)$ cover all possible triangles formed by them.
And our derivation will not lose its generalization in high dimensional spaces, since we does not make use of the vertices but only edges of the triangles.}.
A weighted summation of the augmented loss parts from different augmented data mapping functions can be seen as an empirical estimation of the expectation of the rewritten loss by sampling different  $(r, \theta)$'s from their joint distribution $P$, though the corresponding ground truth $P$ is not observed.

This inspires us how to avoid to specifically design or choose several augmented data mapping functions and their weights used in (\ref{eq:obj}). We can directly design the distribution $P$ of $(r,\theta)$ and optimize the expectation of the rewritten loss (i.e. the right hand side in (\ref{eq:tri})) under this distribution. 
Hence, we propose to optimize the following objective to mimic the effect of data augmentation:
%
\begin{equation}
\ell_{our}\! =\!\ell (f_{x,y}(\x), \y) \!+\! \mathbb{E}_{ (r,\theta) \in P}[ \Phi(\ell (f_{x,y}(\x), \y))] \!\!\!\!\label{eq:ours}
\end{equation}
where $\Phi(e;r,\theta)$ is a function of an edge $e$ in the loss function space given $(r,\theta)$:
\begin{eqnarray}
\Phi(e;r,\theta) = \sqrt{e^2 + r^2 - 2er\cos\theta}.
\end{eqnarray}

\subsection{Optimization}
\label{sec:opt}
We design specific distributions of $(r, \theta)$ used in the proposed objective (\ref{eq:ours}) and their optimization.
We assume the two variables are independent:
\begin{eqnarray}
p(r,\theta)=p(r)p(\theta).
\end{eqnarray}
In the following corollary, we first show the result by assuming that both $r$ and $\theta$ follow uniform distributions. Recall that proper data mapping functions augment samples close to the original one. An ideal case is thus to perturb samples with their output representations uniformly surrounding that of the original sample. The uniform distribution with a small perturbation radius upper bound $R$ can simulate this ideal case.

\begin{corollary}
We are given the perturbation distance upper bound $R$ and assume that 
\begin{eqnarray}
r \sim \mathcal{U}(0,R),\theta \sim \mathcal{U}(0,\pi).
\end{eqnarray}
$\mathbb{E}_{ (r,\theta) \in P}[\Phi(\ell (f_{x,y}(\x), \y))]$ is upper bounded by $\frac{1}{2} \ell (f_{x,y}(\x), \y) + C_1\cdot\ell^2 (f_{x,y}(\x), \y) + C_2(R)$, where $C_1$ is a constant and $C_2(R)$ is another constant dependent on $R$.
\end{corollary}
\noindent
Proof is in the Appendix. With the above result, we can optimize the objective in (\ref{eq:ours}) by minimizing the derived upper bound.  We calculate its gradient:
\begin{eqnarray}
\frac{\partial{\ell_{our}}}{\partial{\Theta}} =
\frac{3}{2}\cdot\frac{\partial{\ell(\Theta)}}{\partial{\Theta}} + 2  C_1 \cdot\ell(\Theta) \frac{\partial{\ell(\Theta)}}{\partial{\Theta}} \label{eq:grad}
\end{eqnarray}
where $\Theta$ contains all neural model parameters.
It can be observed that the major difference of the above gradient compared with the original one of the objective in (\ref{eq:obj}) lies in the second part of (\ref{eq:grad}), which weights the original gradient by the loss value.
This means that the performance improvement brought by data augmentation under our formulation can be equivalently accomplished by
specialized data weighting.
Indeed, many data weighting methods~\cite{lin2017focal} favors hard examples by reducing the gradient contribution from easy examples and increasing the importance of hard examples (example with large loss value in our approach), which significantly boost the performance.  This in turn shows that simple uniform distributions assumed here should be reasonable and effective.


Instead of uniform distribution, we can assume a uniform distribution on $\theta$ but an exponential distribution on $r$ such that a small perturbation distance is preferred with a higher probability.

\begin{corollary}
We are given the expected value of the perturbation distance as $R$ and assume that 
\begin{eqnarray}
r \sim  \mbox{Exp}(\frac{1}{R}),\theta \sim \mathcal{U}(0,\pi).
\end{eqnarray}
$\mathbb{E}_{ (r,\theta) \in P}[\Phi(\ell (f_{x,y}(\x), \y))]$ is upper bounded by $C_1(R) \cdot \ell (f_{x,y}(\x), \y) + \frac{C_1(R)}{2} \cdot \ell^2 (f_{x,y}(\x), \y) + C_2(R)$, where $C_1(R)$ and $C_2(R)$ are constants dependent on $R$.
\end{corollary}
\noindent
Proof is in the Appendix. The above corollary shows that even if different distributions are assumed, we can still use gradient weighting to optimize the proposed objective, where $C_1(R)$ can be set as a hyper-parameter.


If the loss is Lipschitz smooth, of which Euclidean distance is the case, we can prove the convergence of our approach with the convergence rate $O(1/\sqrt{T})$, if SGD is used. The proof is provided in the Appendix, which is extended from results in~\citet{reddi2016stochastic}.
\begin{theorem} \label{thm1}
Suppose $\ell_{our}$ is in the class of finite-sum Lipschitz smooth functions, has $\delta$-bounded gradients, and the weight of the loss gradient is clipped to be bounded by $[w_1, w_2]$. Let the learning rate of SGD $\alpha_t=c/\sqrt{T}$ where $c= \sqrt{\frac{2(\ell_{our}(\Theta^0)-\ell_{our}(\Theta^*))}{L\sigma^2w_1w_2}}$ where $L$ is the Lipschitz constant and $\Theta^*$ is an optimal solution. Then the iterates of SGD of our approach with $\ell_{our}$ satisfy:
\begin{eqnarray}
  &  \lefteqn{\min_{0\le t\le T-1}\mathbb{E}[||\nabla \ell_{our}(\Theta^t)||^2] \leq} \nonumber\\ 
  && \sqrt{\frac{2(\ell_{our}(\Theta^0)-\ell_{our}(\Theta^*))Lw_1}{Tw_2}}\sigma. 
\end{eqnarray}
\end{theorem}

\subsection{Other Types of Augmented Data} 
\label{sec:types}
We now discuss how our approach can be applied to other types of augmented data.
For augmented data on the output domain, the objective in (\ref{eq:obj}) becomes:
\begin{equation}
\ell_{aug} = \ell (f_{x,y}(\x), \y) + \sum_{\phi_{\hat{y}} \in \mathcal{F}} w_{\hat{y}} \ell (f_{x,y}(\x), \hat{\y}). \label{eq:obj2}
\end{equation}
The augmented loss part can be rewritten using the polar coordinate system with $\y$ as the pole and $(\y,f_{x,y}(\x))$ as the polar axis, illustrated in Fig.~\ref{fig2}:
\begin{eqnarray}
\ell^2 (f_{x,y}(\x), \hat{\y})&=& \ell^2(\y, f_{x,y}(\x)) + \ell^2(\y, \hat{\y}) \nonumber\\
&&-2\ell(\y, f_{x,y}(\x))\ell (\y, \hat{\y})\cos\theta. \nonumber\\
\label{eq:tri2}
\end{eqnarray}
Similarly, the augmented data mapping function $\phi_{\hat{y}}$ can be re-parameterized into a function of the radius $r=\ell(\y, \hat{\y})$ (still the perturbation distance) and the radian of $\hat{\y}$.
The objective turns out to be the same as (\ref{eq:ours}).

For data perturbation on both the input and output space, we have:
\begin{equation}
\ell_{aug} = \ell (f_{x,y}(\x), \y) + \!\!\!\sum_{\phi_{\hat{x}, \hat{y}} \in \mathcal{F} }\!\!\! w_{\hat{x},\hat{y}} \ell (f_{x,y}(\hat{\x}), \hat{\y}). \!\label{eq:obj3}
\end{equation}
Illustrated in Fig.~\ref{fig3}, we first make use of the triangle inequality that:
\begin{eqnarray}
\ell (f_{x,y}(\hat{\x}), \hat{\y}) \!\!\!\!\!\!\!\!&& \leq \frac{1}{2}(\ell (f_{x,y}(\hat{\x}), \y) + \ell (\y, \hat{\y})) \nonumber\\
&& \!\!\!\!\!\!\!\!\!\!\!\!\!\!\!\!\!\!\!\!\!\!\!+\frac{1}{2}(\ell (f_{x,y}(\hat{\x}), f_{x,y}(\x)) + \ell (f_{x,y}(\x), \hat{\y})). \nonumber\\
\end{eqnarray}
Using (\ref{eq:tri}) and (\ref{eq:tri2}),  the objective is rewritten as:
\begin{eqnarray}
\ell_{our} &=&\ell (f_{x,y}(\x), \y)  \nonumber \\
&&+\mathbb{E}_{ (r,\theta) \in P}[ r+ \Phi(\ell (f_{x,y}(\x), \y))].\nonumber \\ 
\end{eqnarray}
Note that $\mathbb{E}_{ (r,\theta) \in P}[ r]$ is a scalar which is not dependent on any learning parameter. Thus optimizing the above objective is equivalent to optimizing (\ref{eq:ours}).

From the above analysis, we can see that our proposed objective in (\ref{eq:ours}) can be applied to handle all three kinds of augmented data mapping functions in text generation models.

\section{Loss Function}
\label{sec:loss}
In theory, our approach can be applied to any Lipschitz smooth loss function that holds the equation (\ref{eq:tri}). 
In this section, we show another valid loss function in our approach -- the word mover's distance (WMD)~\cite{kusner2015word,zhao2019moverscore}, which is previously used in various text generation tasks.
Next, we discuss the cross entropy loss, in which the proposed objective is not an upper-bound of the data augmentation objective. However, our approach can still converge with the same convergence rate and experimental results in the next section validate the effectiveness of our approach with the cross-entropy loss.  


\subsection{Word Mover's Distance}
\label{sec:wmd}
WMD, also named the optimal transport distance~\cite{chen2018improving}, leverages optimal transport to find an optimal matching of similar
words between two sequences, providing a way to measure their semantic similarity:
\begin{eqnarray}
&\ell_{WMD}(\seq{u},\seq{v})  = & \min_{T_{i,j}} \sum_{i,j}T_{i,j} d_{i,j} \\
&\mbox{s.t.} &\sum_{j=1}^M T_{i,j} = p_{u,i} \quad \forall i \nonumber\\
&&\sum_{i=1}^N T_{i,j} = p_{v,j} \quad \forall j \nonumber
\end{eqnarray}
 where $p_{u,i}/p_{v,j}$ is the probability distribution of the sentence, i.e. $\sum_i p_{u,i} = 1$ and $\sum_j p_{v,j} = 1$. $d_{i,j}$ is the cost for mis-predicting $u_i$ to $v_j$, where the squared Euclidean distance $d_{i,j} = \|u_i - v_j\|^2$ is used and $u_i/v_j$ is the word embedding vector. 
 Note that the Euclidean distance in (\ref{eq:l2}) is a special case of WMD by replacing the 1-gram used in WMD to $n$-gram with $n$ larger than the sentence's length. 
 WMD is the squared $L^2$ Wasserstein distance. We take its squared root, i.e. $\ell_{WD} = \sqrt{\ell_{WMD}}$, which holds an upper bound as the right hand side in (\ref{eq:tri}). Also, $\ell_{WD}$ is Lipschitz smooth.
 \begin{theorem}\label{thm-cos-law}
For the $L^2$ Wasserstein distance $W_2(\cdot, \cdot)$ on the Wasserstein space $W^2(\mathbb R^n)$ and any $x,y,z \in W^2(\mathbb R^n)$, we have
\begin{eqnarray}
W_2(y, z)^2 \leq W_2(x, y)^2 + W_2(z, x)^2 \nonumber\\
- 2 \cdot W_2(x, y) \cdot W_2(z, x) \cdot \cos \theta.
\end{eqnarray}
Here $\theta$ is the angel between the $\gamma_{xy}$ and $\gamma_{zx}$, $\gamma_{xy}$ is the geodesic (shortest path) connecting $x,y$ in $W^2(\mathbb R^n)$, and $\gamma_{zx}$ is the geodesic connecting $z,x$ in $W^2(\mathbb R^n)$. 
\end{theorem}
\begin{theorem}
$\seq{u}$ and $\boldsymbol{v}$ are given as fixed. Assuming that $\boldsymbol{u}_\Theta$ is Lipschitz continuous with respect to the parameters $\Theta$. 
Then $
\ell_{WD} (\boldsymbol{u}_\Theta, \boldsymbol{v}) 
$
is Lipschitz continuous with respect to the parameters $\Theta$. 
\end{theorem}
\noindent
Roughly speaking, according to \citet{sturm2006geometry}[Proposition 2.10], the sectional curvature of Wasserstein space $W^2(\mathbb R^n)$ is non-negative. Hence, every geodesic triangle in $W^2(\mathbb R^n)$ is fatter than the one with same sides length in $\mathbb R^2$. As a consequence, an inequality like cosine law is satisfied on $W^2(\mathbb R^n)$, i.e., Theorem \ref{thm-cos-law} holds.  A formal proof of the above two theorems is provided in the Appendix.  Thus, all our derivations in Section.~\ref{sec:model} hold.

 The exact computation of $\ell_{WD}$ is expensive during training. In our experiments, we resort to the inexact proximal point method for optimal transport algorithm to compute it~\cite{chen2018improving}.

\subsection{Cross-entropy Loss}
Although WMD is effective for various sequence generation tasks, 
the most conventional loss function adopted in existing generation models is the cross-entropy loss. It measures the word difference at each word $y_i$
of the output sequence $\y$:
\begin{eqnarray}
\ell_{CE}(\y_i,\seq{p}_i) &=&  \y_i^T \log (\seq{p}_i)  \label{eq:cross-entropy}\\
\ell_{CE}(\y,\seq{p}) &=& \sum_{i=1}^{|\y|} \ell_{CE}(\y_i,\seq{p}_i)
\end{eqnarray}
where $\y_i$ is the target one-hot vector with the correct dimension as 1 and 0 elsewhere, and $\seq{p}_i$ is the predicted probability output by a softmax layer. We adopt the maximum likelihood estimation as the training paradigm by assuming truth for preceding words in predicting $\seq{p}_i$.


The cross-entropy loss is also Lipschitz smooth, and thus we can guarantee its convergence from Theorem~\ref{thm1}.
Unfortunately, it does not satisfy the equation in (\ref{eq:tri}), and thus minimizing our objective in (\ref{eq:ours}) does not necessarily approximate the data augmentation objective in (\ref{eq:obj}).
In our experiments, we also try the cross-entropy loss, and results show that our objective is effective to improve the model performance compared with the base model. This is not surprising since our approach is optimized by gradient weighting and thus at least it is a useful data weighting method. 
%

\section{Experiments}

The proposed approach provides a new paradigm and understanding of data augmentation for text generation. To evaluate that our approach can mimic the effect of data augmentation, we conduct experiments on two text generation tasks -- neural machine translation and conversational response generation.
We compare our approach with two most popular data augmentation methods (one token-level and one sentence-level augmentation method)  that can be applied on various text generation tasks: 
\begin{itemize}[wide=0\parindent,noitemsep, topsep=0pt]
    \item Masked Language model (MLM): We use a pre-trained BERT~\cite{devlin2019bert,wolf-etal-2020-transformers} and randomly choose 15\% of the words for each sentence. BERT takes in these masked words to predict these masked positions with new words. We augment one sample from each original training sample. Thus the data size increases to twice of the original one. Note that we only augment the English side of translation datasets.
    \item Back-translation (BT): For neural machine translation, we employ a fixed target-to-source translation model trained on the original dataset.  For conversational response generation, we perturb both the input and output text of the original sample pair using two pretrained translation model: an English-to-German model and its backward counterpart,  which are obtained using the WMT14 corpus with 4.5M sentence pairs\footnote{Datasets used in this work can be found at \url{https://nlp.stanford.edu/projects/nmt/, http://coai.cs.tsinghua.
edu.cn/hml/dataset/\#commonsense
}}. We again augment one sample from each original training sample.
\end{itemize}

\begin{table*}[t]
\begin{center}
\resizebox{2.05\columnwidth}{!}{
\begin{tabular}{lcccccccc}\toprule
 \hline
\textbf{Model} &  \textbf{De$\Rightarrow$En} & \textbf{En$\Rightarrow$De} & \textbf{Vi$\Rightarrow$En} & \textbf{En$\Rightarrow$Vi}  
 & \textbf{Fr$\Rightarrow$En} & \textbf{En$\Rightarrow$Fr} & \textbf{It$\Rightarrow$En} & \textbf{En$\Rightarrow$It}\\\hline
\textsc{CE}     & 27.98   & 22.85 & 24.22   & 27.09 & \underline{40.49} & 40.86 & 29.70 & 26.85
 \\
\textsc{CE+MLM}   &  28.70  & 23.23 & 24.40  & 26.20 & 40.03 & 40.79 & 29.35 &	26.90 \\
\textsc{CE+BT}     & \textbf{29.35}  & \textbf{24.09} & \textbf{25.00}   & \textbf{27.41} & \textbf{40.87} & \textbf{42.64} & \textbf{30.44} & \textbf{27.94}\\
\textsc{CE+OURS}    & \underline{29.16}   & \underline{23.26}  & \underline{24.74}  & \underline{27.12} & 40.46 & \underline{40.94} &
 \underline{29.79} & \underline{27.11} \\\hline
\textsc{WD}  &  28.53  &  22.95 &  24.03 &  26.69 & 39.71 & 40.48 & 29.74 & 27.08
  \\
\textsc{WD+MLM}  & \underline{28.80} & 22.98 & \underline{24.33} &  \textbf{26.88} & 39.57 & \underline{40.61} & \textbf{29.98} & 26.59  \\
\textsc{WD+BT}   & 28.56   &  \underline{23.10}   & \textbf{24.51}  &  \underline{26.74} & \underline{39.77} & 40.60		 & 29.56 & \textbf{27.33}\\
\textsc{WD+Ours}  & \textbf{28.91}   & \textbf{23.42} & 24.26   & 26.73 & \textbf{40.46} & \textbf{41.07} & \underline{29.86} & \underline{27.15}
\\\hline
\bottomrule
\end{tabular}
}
\end{center}
\caption{\label{tab:translation1} BLEU scores on various translation datasets.  CE: Cross-Entropy loss; WD: $L^2$ Wasserstein distance. The best results are in \textbf{bold}, and the second-best results are in \underline{underline}.}
\end{table*}

We set the same weight $w$ of all augmented loss parts used in $\ell_{aug}$ as a hyper-parameter, and tune it on the development set of each dataset.
Since Euclidean distance is a special case of WMD as discussed in Sec~\ref{sec:wmd}, we show results of all methods with the use of the cross-entropy loss and WD.
We mainly use the Fairseq~\cite{ott2019fairseq} Seq2seq implementation as our model.
Both encoder and decoder are one-layer LSTM. The word embedding dimension is 256.  Attention~\cite{luong2015effective} is used with a dropout rate of 0.1. All parameters are randomly initialized based on the uniform distribution $[-0.1, +0.1]$. We use SGD to optimize our models, and the learning rate is started with 1.0. After 8 epochs, we start to halve the learning rate after each epoch. 
All experiments are run on a single NVIDIA V100 GPU. 
Code for our experiments are available once our work is accepted.

\subsection{Neural Machine Translation}
We use translation benchmarks IWSLT14 En–De, En–Fr, En–It, and IWSLT15 En–Vi in our experiments. 
The datasets of IWSLT14 are pre-processed with the script in Fairseq~\footnote{\url{https://github.com/pytorch/fairseq/blob/master/examples/translation/prepare-iwslt14.sh}}. 
For IWSLT14 datasets, we use tst2011 as validation set and tst2012 as test set. The IWSLT15 dataset is the same as that used in \citet{luong2015stanford}, and the validation and test sets are tst2012 and tst2013, respectively.

Table~\ref{tab:translation1} shows the BLEU scores on their test sets. For both cross-entropy loss and $L^2$ Wasserstein distance, all data augmentation methods (MLM, BT and OURS) perform better than the corresponding base models in most cases. 
The improvement margins are different across the various datasets. The reason may be that the datasets are in different scales and the alignment difficulty between different languages can also vary. 
The performance of MLM is not stable from our results, which is largely due to that masked tokens are possible to be filled in with different semantic ones and thus the semantics of the sentence changes. Therefore, the augmented data are not aligned indeed, and the translation model learning can be distracted. Note that we also evaluate our method using the Transformer model and get some similar findings. Experimental results of the Transformer model are presented in the appendix.

Compared to BT and MLM, our approach that mimics the effect of data augmentation without actually constructing augmented samples, shows encouraging results. 
Note that our proposed objective may not have a theoretical guarantee on the cross-entropy loss. Yet, it still manages to improve the base model except for Fr$\Rightarrow$En, and surpasses MLM on all datasets. With the use of $L^2$ Wasserstein distance, our approach even outperforms BT and achieves the best performance on half test sets. This validates the benefits of not using any specific data augmentation mapping function in data augmentation as in our proposed objective.

\begin{figure}[t]
  \centering
  \includegraphics[width=\linewidth]{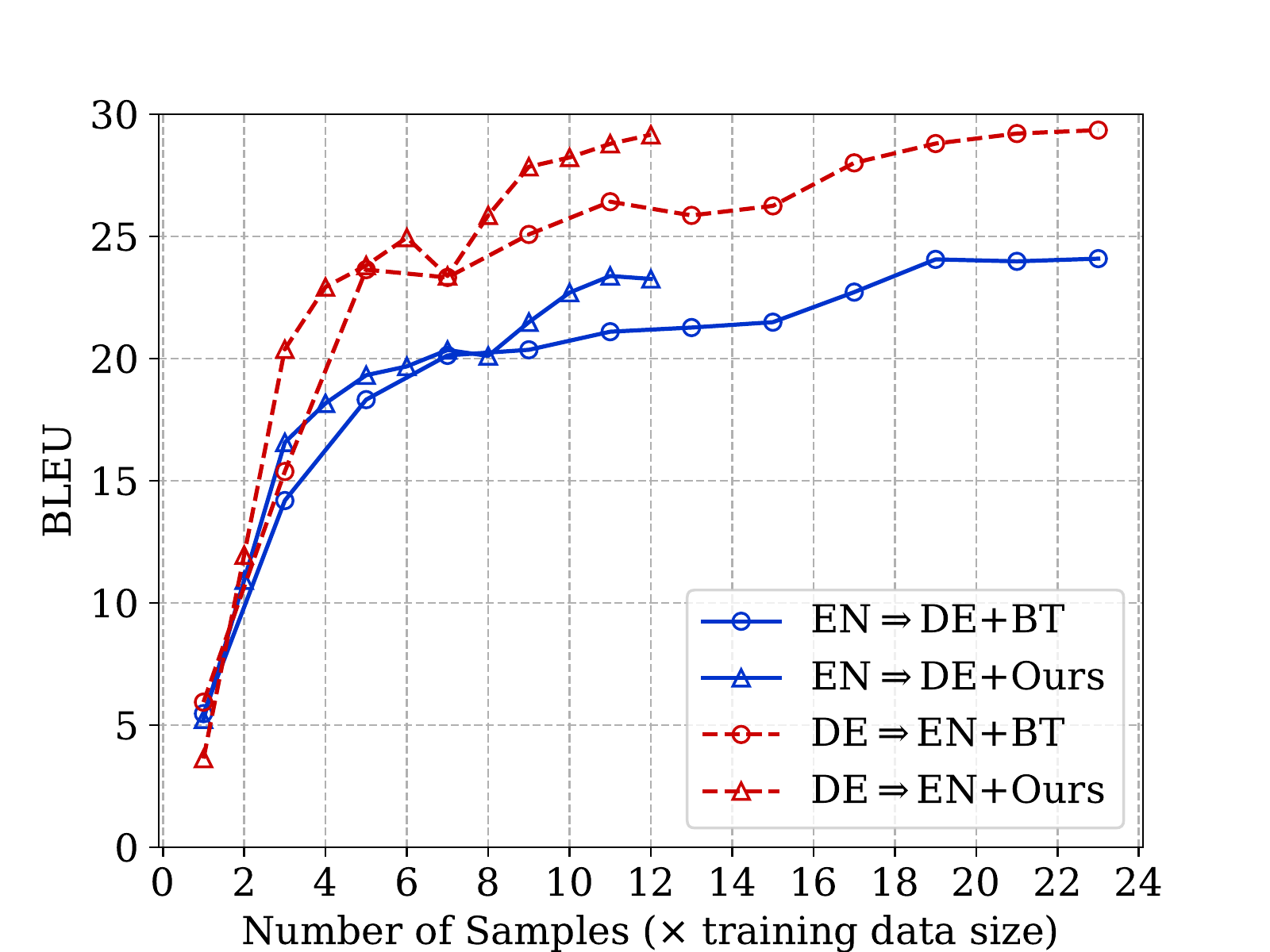}
  \caption{BLEU scores by models updated with the same number of samples. }
  \label{fig:update}
\end{figure}

\begin{figure}[t]
  \centering
  \includegraphics[width=\linewidth]{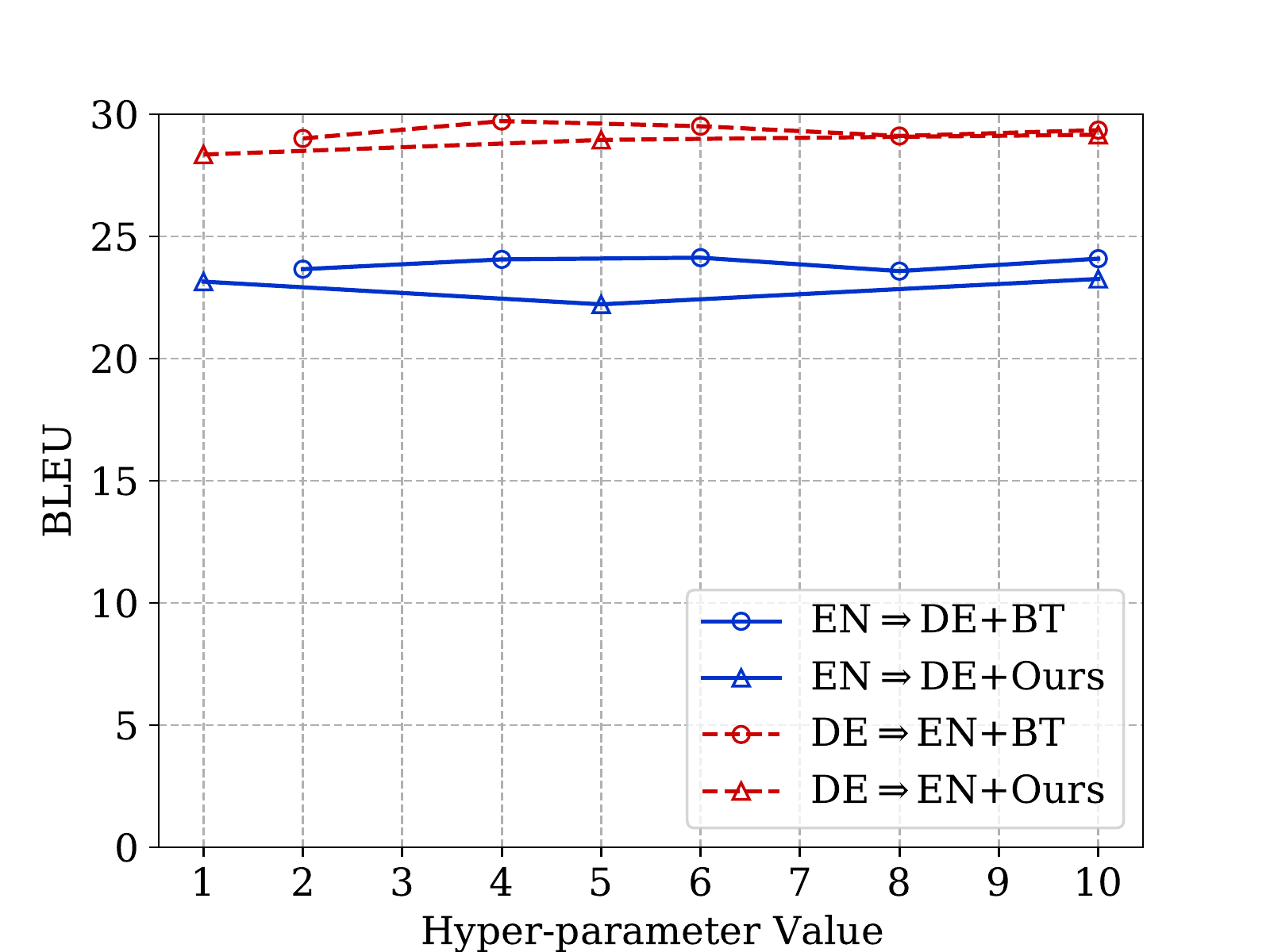}
  \caption{BLEU scores by models trained with different hyper-parameters. Values in the x-axis are re-scaled in order to visualize them in the same range.}
  \label{fig:hyper-para}
\end{figure}

\begin{table*}[t]
\begin{center}
\resizebox{2.05\columnwidth}{!}{
\begin{tabular}{lcccccc|cc}\toprule
\hline
\textbf{Model} 
&  PPL & BLEU  &  BLEU-1 & BLEU-2 & Dist1 & Dist2 &  Flu & Rel \\\hline
\textsc{CE}   & 7.22 &  0.75 & 16.35 & \underline{1.38}  & 0.889 & 0.855 
  & \underline{3.571} & \underline{3.314} \\

\textsc{CE+MLM} & \textbf{6.82} &	\underline{0.76} &	\underline{16.65} & 1.31 & \textbf{0.917} & \textbf{0.868} & 3.552 & 3.184
  \\
\textsc{CE+BT}   &  7.38  & 0.68 & \textbf{17.04}  &  1.33 & 0.892 & 0.851 & 3.557 & 3.249
 \\
\textsc{CE+OURS}   &  \underline{7.10}  & \textbf{0.85} & 16.41  & \textbf{1.44} & \underline{0.894} & \underline{0.864} & \textbf{3.632}  & \textbf{3.370}
 \\\hline
\textsc{WD} &  7.10 &  \textbf{0.87} & 15.09 &  \underline{1.33}  & 0.872 & \underline{0.863} & \textbf{3.644} &  \underline{3.354} \\
\textsc{WD+MLM}  &  7.09 & 0.57 & 15.75 & 1.25 & \textbf{0.913} & \textbf{0.881} & 3.575  & 3.188
    \\
\textsc{WD+BT} & \textbf{6.92}  &  0.81 &  \underline{15.97}  &  1.29 & 0.881 & 0.853 & 3.579 & 3.279
 \\
\textsc{WD+Ours} & \underline{7.01} & \underline{0.84}  & \textbf{16.56} & \textbf{1.39}  & \underline{0.893} & 0.855 & \underline{3.629} & \textbf{3.447}
  \\
\hline
\textsc{Human}  & - &  - & - & -  & 0.947 & 0.897 & 4.235 & 4.086 \\ 
\hline
\bottomrule
\end{tabular}
}
\end{center}
\caption{\label{tab:dialogue} Automatic and human evaluation results on Reddit. Human: the gold reference of the query. The best results are in \textbf{bold}, and the second-best results are in \underline{underline}.}
\end{table*}

We provide further analysis on the performance of our approach versus BT.
In Fig.~\ref{fig:update}, we compare testing BLEU scores obtained by models updated with the same number of samples. Since we construct one augmented sample from each original training sample, the total number of samples used in BT is twice as much as that of our approach. 
We can see that our approach achieves compatible performance with BT, while only requires half of the training data. This shows that our approach, without involving additional calculations on extra samples, can effectively save the computational expense. Fig.~\ref{fig:hyper-para} shows the sensitivity of performance under different hyper-parameters. For our approach, we vary across different $C_1(R)$'s; for BT, we vary the sample weight $w$ of the augmented samples. We re-scale $C_1(R)$ by $10^{-4}$ and $w$ by $10^{-1}$, in order to visualize them within the same range of x-axis. 
Both BT and our approach demonstrate their robustness under different settings of their hyper-parameters.



\subsection{Conversational Response Generation}
We use the English single-round Reddit conversation dataset~\cite{Zhou2018CommonsenseKA}. Following previous work on data augmentation for dialogue system~\cite{cai2020,zhang2020dialogue}, we simulate a low data regime so that data augmentation is expected to be more effective.
Thus, we select data pairs with the length of both the query and response less than 20, and randomly split them into 200K for training, 2K for validation and 5K for testing.
Automatic evaluation for each method is performed on all test data. We report Perplexity, BLEU and BLEU-k (k=1,2) to measure the response coherence; Distinct-k (k=1,2)~\cite{li_diversity} to measure the response diversity.
We also hire five annotators from a commercial annotation company for manual evaluation on 200 pairs randomly sampled from the test set.
Results of all methods are shuffled for annotation fairness. Each annotator rates each response on a 5-point scale (1: not acceptable; 3: acceptable; 5: excellent; 2 and 4: used in unsure case) from two perspectives:  Fluency and Relevance.

Results are summarized in Table~\ref{tab:dialogue}.
On automatic metrics, BT only shows marginal improvements on a few metrics, which can not exhibit its strength as in translation tasks. 
MLM effectively increases the response diversity (Dist1\&2). This is due to nature of the conversation data that conversation pair often remains coherent even if the semantics of the query or response has been slightly changed. Thus, MLM can increase data diversity, which is appreciated in training response generation models.
In terms of human evaluation, BT and MLM can barely improve the base model.
As for our approach, it achieves the best or second best results on most metrics for both loss functions,  demonstrating more robust performance than BT and MLM.
This is consistent with our statement in the introduction that we often need to design proper augmented data mapping functions carefully for a target generation task, which requires non-trivial work. As such, it is meaningful to avoid the use of specific data augmentation techniques and find a unified formulation of data augmentation for general generation tasks.
From our results, the proposed objective demonstrates its power to achieve the effect of data augmentation across different generation tasks.

\section{Conclusions and Future Work}
We have proposed an objective of formulating data augmentation without any use of any augmented data mapping function.
We show its optimization and provide the corresponding convergence rate. 
Both the $L^2$ Wasserstein distance and the cross-entropy loss are discussed with their use in our objective and their corresponding theoretical guarantees. 
Different from previous data augmentation works that need to add manipulated data into the training process, our gradient based approach provides a potential way to obtain performance improvements, which may come from augmented data, without incurring the computational expense. 
Experiments on both neural machine translation and conversational response generation validate  the effectiveness of our objective compared to existing popular data augmentation methods: masked language models and back-translation.

We believe this work provides a new understanding of data augmentation. Our approach can also be useful to a wide range of tasks including 
text classification tasks, which can be seen as special cases of text generation tasks, and
cross-modality generation tasks such as image captioning, in which we can skip the step to use various image augmentation techniques. 
%

We would like to point out that some parts of our approach can be improved in the future, which may lead to a better performance and generalization. Firstly, current distributions we choose in the re-parameterized loss are relatively simple. Some points under current continuous distributions may not correspond to valid text sequences in the original text space, due to the discreteness of natural languages. A possible way is that we change to leverage more informative distributions, such as including prior distributions computed from several augmented samples. Secondly, our method is derived under the framework of SGD and it is possible to extend it to the Adam framework~\cite{kingma2014adam, chen2018convergence, reddi2019convergence}. We also leave the more general version of our work in the future.

\bibliography{acl2021}
\bibliographystyle{acl_natbib}

\onecolumn
\appendix

\section{Proof of Corollary 1}

 \begin{eqnarray}
&&E_{ (r,\theta) \in P}[ \sqrt{L^2 + r^2 - 2Lr\cos\theta}] \nonumber\\
&=& \int_{r=0}^{R}\int_{\theta=0}^{\pi} \frac{1}{R} \cdot \frac{1}{\pi} \cdot \sqrt{L^2 + r^2 - 2Lr\cos\theta} \mathrm{d}r \mathrm{d}\theta, \nonumber\\
& = & \int_{r=0}^{R} \frac{1}{R} \cdot \frac{1}{\pi} (\int_{\theta=0}^{\pi/2} \sqrt{L^2 + r^2 - 2Lr\cos\theta} \mathrm{d}\theta +\int_{\theta=\pi/2}^{\pi} \sqrt{L^2 + r^2 - 2Lr\cos\theta} \mathrm{d}\theta) \mathrm{d}r \nonumber\\
& \leq & \int_{r=0}^{R} \frac{1}{R} \frac{1}{2} (\sqrt{L^2 + r^2} + L + r) \mathrm{d}r \nonumber\\
&=& \frac{1}{2} L + \frac{R}{4} +  \frac{1}{2R}\int_{r=0}^{R} \sqrt{L^2 + r^2} \mathrm{d}r \nonumber\\
&\leq& \frac{1}{2} L + \frac{R}{4} +   \frac{1}{2R}\int_{r=0}^{R} \frac{1 + L^2 + r^2}{2} \mathrm{d} r \nonumber \\
&=&  \frac{1}{2} L + L^2 C_1 + C_2(R).  
\end{eqnarray}
\noindent
where $L=\ell (f_{x,y}(x), y)$, $C_1 =\frac{1}{4}$, $C_2(R)=\frac{R^2}{12} + \frac{R}{4} + \frac{1}{4}$. 

\section{Proof of Corollary 2}
\begin{eqnarray}
&&\int_{r=0}^{\infty}\int_{\theta=0}^{\pi} \frac{1}{R}\exp(-\frac{r}{R}) \frac{1}{\pi}(\sqrt{L^2 + r^2- 2Lr\cos\theta} \mathrm{d}r \mathrm{d}\theta \nonumber\\
& \leq & \int_{r=0}^{R} R\exp(-\frac{r}{R}) \frac{1}{2} (\sqrt{L^2 + r^2} + L + r) \mathrm{d}r \nonumber\\
&=& \int_{r=0}^{R} R\exp(-\frac{r}{R}) \frac{1}{2} (L + r) \mathrm{d}r  + \int_{r=0}^{R} R\exp(-\frac{r}{R}) \frac{1}{2} (\sqrt{L^2 + r^2}) \mathrm{d}r \nonumber\\
&=& \frac{R^2}{2} (1-\mathrm{e}^{-1}) L + \frac{R^3}{2} (1-2 \mathrm{e}^{-1}) + \int_{r=0}^{R} R\exp(-\frac{r}{R}) \frac{1}{2} (\sqrt{L^2 + r^2}) \mathrm{d}r \\
&\leq& \frac{R^2}{2} (1-\mathrm{e}^{-1}) L + \frac{R^3}{2} (1-2 \mathrm{e}^{-1}) + \int_{r=0}^{R} R\exp(-\frac{r}{R}) \frac{1+L^2+r^2}{4}\mathrm{d}r \\
&=& LC_1(R) + L^2\frac{C_1(R)}{2} + C_2(R)
\nonumber
\end{eqnarray}
\noindent
where $C_1(R)=(1-e^{-1})\frac{R^2}{2}$,
and $C_2(R) = \frac{R^3}{2} + \frac{3R^2}{4} - (\frac{R^3}{2} + \frac{R^4}{2}) e^{-1}$. 

\section{Proof of Theorem 1}

We study the nonconvex \emph{finite-sum} problems of the form
\begin{equation}
\min_{\Theta} \mathcal{L}(\Theta) := \frac{1}{n}\sum_{i=1}^{n}\ell_{our}(\Theta,x_i,y_i),
\label{obj}
\end{equation}
where both $\mathcal{L}$ and $\ell_{our}$ may be nonconvex. 
For ease of notation, we use $\ell$ to denote $\ell_{our}$ in the following of the proof.
We denote the class of such finite-sum Lipschitz smooth functions by $\mathcal{F}_n$.
We optimize functions in $\mathcal{F}_n$ with the gradient in Eq.~8 by SGD.
For $\mathcal{L} \in\mathcal{F}_n$, SGD takes an index $i\in[n]$ and a sample in the training set, and returns the pair $(\ell_i(\Theta),\nabla \ell_i(\Theta))$.

\begin{definition}
We say $\mathcal{L}:\mathbb{R}^d\to \mathbb{R}$ is $L$-smooth if there is a constant $L$ such that
\begin{equation}
||\nabla \ell(\Theta') - \nabla \ell(\Theta)||\le L||\Theta'-\Theta||,\forall \Theta',\Theta\in \mathbb{R}^d.
\end{equation}
\end{definition}

\begin{definition}
A point $\Theta$ is called $\epsilon$-accurate if $||\nabla \ell(\Theta)||^2\le\epsilon$.
A stochastic iterative algorithm is said to achieve $\epsilon$-accuracy in $t$ iterations if $\mathbb{E}[||\nabla \ell(\Theta^t)||^2]\le \epsilon$, where the expectation is over the stochasticity of the algorithm.
\end{definition}

\begin{definition}
We say $\ell \in\mathcal{F}_n$ has $\sigma$-bounded gradients if $||\nabla \ell_i(\bm{\theta})||\le\sigma$ for all $i\in[n]$ and $\Theta\in\mathbb{R}^d$.
\label{bound}
\end{definition}

Let $\alpha_t$ denote the learning rate at iteration $t$, and $w_{i_t}$ be the gradient weight assigned to sample $i$ by our approach.
By SGD, we have
\begin{eqnarray}
\Theta^{t+1} = \Theta^t - \alpha_tw_{i_t}\nabla \ell_{i_t}(\Theta^t), i\in [n].
\label{iter}
\end{eqnarray}

\begin{definition}
We say the positive gradient weight $w$ in our approach is bounded if there exist constants $w_1$ and $w_2$ such that $w_1 \le w_i \le w_2$ for all $i\in[n]$.
\label{iw}
\end{definition}


\begin{proof}[Proof of Theorem1]
According to the Lipschitz continuity of $\nabla \ell$, the iterates of our approach satisfy the following bound:
\begin{eqnarray}
\mathbb{E}[\ell({\Theta}^{t+1})] \le \mathbb{E}[\ell({\theta}^t) + \langle\nabla \ell({\Theta}^t),{\Theta}^{t+1}-{\Theta}^t\rangle + \frac{L}{2}||{\Theta}^{t+1}-{\Theta}^t||^2].
\label{Lcon}
\end{eqnarray}
After substituting~(\ref{iter}) into~(\ref{Lcon}), we have:
\begin{eqnarray}
\mathbb{E}[\ell({\Theta}^{t+1})] 
\!\!\!\!\!\!&& \le \mathbb{E}[\ell({\Theta}^t)] - \alpha_t w_t\mathbb{E}[||\nabla \ell({\Theta}^t)||^2] + \frac{L\alpha_t^2w_t^2}{2}\mathbb{E}[||\nabla \ell_{i_t}({\Theta}^t)||^2] \nonumber \\
\!\!\!\!\!\!&& \le \mathbb{E}[\ell({\Theta}^t)] - \alpha_tw_t\mathbb{E}[||\nabla \ell({\Theta}^t)||^2] + \frac{L\alpha_t^2w_t^2}{2}\sigma^2. \label{f}
\end{eqnarray}
The first inequality follows from the unbiasedness of the stochastic gradient $\mathbb{E}_{i_t}[\nabla \ell_{i_t}({\Theta}^t)]=\nabla \ell({\Theta}^t)$.
The second inequality uses the assumption on gradient boundedness in Definition~\ref{bound}.
Re-arranging~(\ref{f}) we obtain
\begin{eqnarray}
\mathbb{E}[||\nabla \ell({\Theta}^t)||^2]\le \frac{1}{\alpha_tw_t}\mathbb{E}[\ell({\Theta}^t)-\ell({\Theta}^{t+1})] + \frac{L\alpha_tw_t}{2}\sigma^2.
\label{f2}
\end{eqnarray}
Summing~(\ref{f2}) from $t=0$ to $T-1$ and using that $\alpha_t$ is a fixed $\alpha$, we obtain
\begin{eqnarray}
\min_t\mathbb{E}[||\nabla \ell({\Theta}^t)||^2]
\!\!\!\!\!\!&& \le \frac{1}{T}\sum_{t=0}^{T-1}\mathbb{E}[||\nabla \ell({\Theta}^t)||^2] \nonumber \\
\!\!\!\!\!\!&& \le \frac{1}{T}\sum_{t=0}^{T-1}\frac{1}{\alpha w_t}\mathbb{E}[\ell({\theta}^t)-\ell({\theta}^{t+1})] + \frac{1}{T}\sum_{t=0}^{T-1}\frac{L\alpha w_t}{2}\sigma^2 \nonumber \\
\!\!\!\!\!\!&& \le \frac{1}{T\alpha w_2}\left(\ell({\Theta}^0-\ell({\Theta}^T)\right) + \frac{L\alpha w_1}{2}\sigma^2 \nonumber \\
\!\!\!\!\!\!&& \le \frac{1}{T\alpha w_2}\left(\ell({\Theta}^0-\ell({\Theta}^*)\right) + \frac{L\alpha w_1}{2}\sigma^2 \nonumber \\
\!\!\!\!\!\!&& \le \frac{1}{\sqrt{T}}\left(\frac{1}{cw_2}(\ell({\Theta}^0)-\ell({\Theta}^*)) + \frac{Lcw_1}{2}\sigma^2\right).
\end{eqnarray}
The first step holds because the minimum is less than the average. 
The second step is obtained from~(\ref{f2}). 
The third step follows from the assumption on gradient weight boundedness in Definition~\ref{iw}.
The fourth step is obtained from the fact that $\ell({\Theta}^*)\le \ell({\Theta}^T)$.
The final inequality follows upon using $\alpha=c/\sqrt{T}$.
By setting
$c = \sqrt{\frac{2(\ell({\Theta}^0)-\ell({\Theta}^*))}{L\sigma^2 w_1w_2}}$
in the above inequality, we get the desired result. 
\end{proof}

\section{Proof of $\ell_{WD}$}
%


We begin with some concepts in mathematics. 
Let $(X,|\,\cdot\,,\cdot\,|)$ be a complete metric space. 
\begin{definition}
\label{defn-geod-space}
A rectifiable curve $\gamma(t) : I \subset \mathbb R^+ \to X$ connecting two points $p,q$ is called a
 \emph{geodesic} if its length is equal to $|p,q|$ and it has unit speed. Here, we say that $\gamma(t) : I \to X$ has unit speed, if for any $s, t \in I$, $s<t$, we have, the length of the restriction $$\gamma : [s, t] \to X$$ is $t-s$.
 A metric space $X$ is called a \emph{geodesic space}
  if, for every pair of points $p,q\in X$, there exists some geodesic connecting them.
 \end{definition}
  
\begin{definition}\label{defn-alex}
We say that, a geodesic space $(X, |\cdot\,,\cdot|)$ has non-negative curvature in the sense of Alexandrov, if it satisfies the following property:
\begin{itemize}
  \item for any $p\in X$, and 
 for any unit speed geodesics $\gamma(s) : I \to X$ and $\sigma(t) : J \to X$ with $\gamma(0)=\sigma(0):=p$, the comparison angle
 $${\widetilde\angle}\gamma(s)p\sigma(t):=\arccos \left(  \frac{t^2 + s^2 - |\gamma(s), \sigma(t)|^2}{2\cdot s \cdot t} \right)$$
is non-increasing with respect to each of the variables $t$ and $s$.
\end{itemize}
The angle between $\gamma $ and $\sigma$ at $p$ is defined by 
$$
\lim_{s,t \to 0^+}\arccos \left(  \frac{t^2 + s^2 - |\gamma(s), \sigma(t)|^2}{2\cdot s \cdot t} \right) \in [0, \pi].
$$
In other words, every geodesic triangle in $X$ is fatter than the one with sides length in $\mathbb R^2$ (Figure \ref{pic-alex}). 
\end{definition}

\begin{figure}[htbp]
\centering
\includegraphics[width=3in]{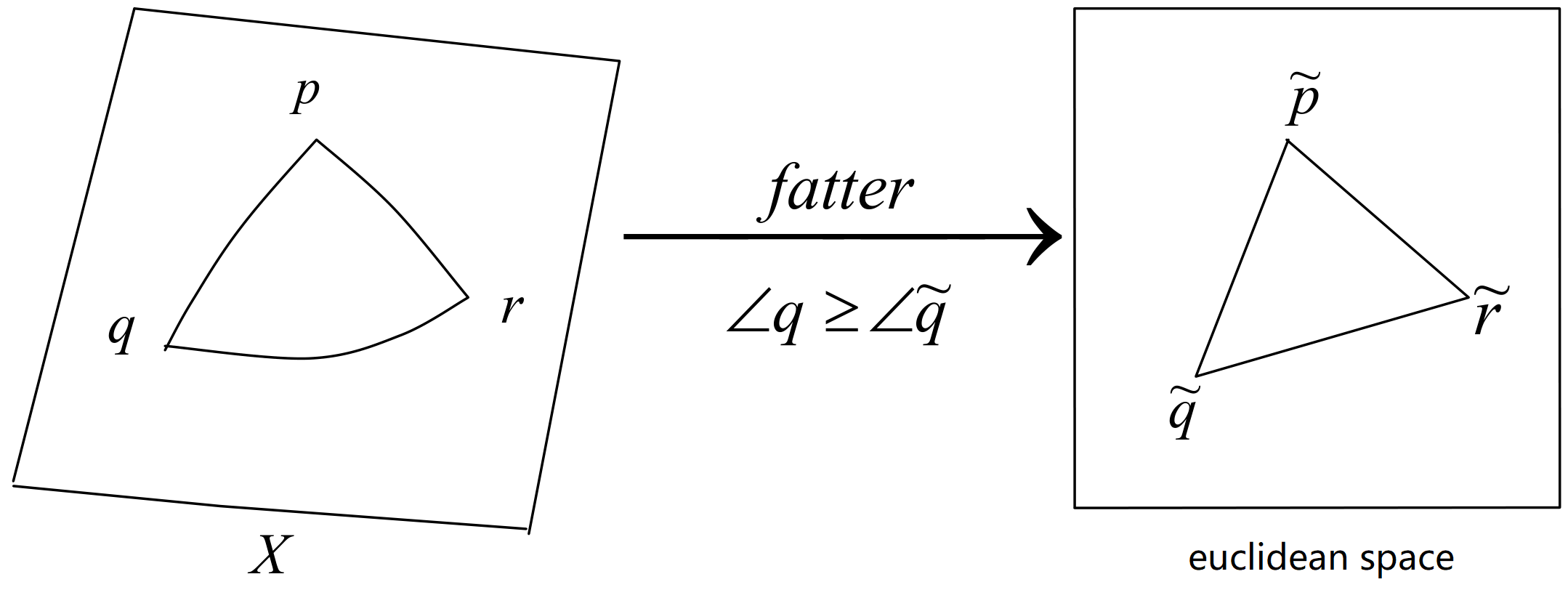}
\caption{geodesic space with non-negative curvature}
\label{pic-alex}
\end{figure}

According to \citet{sturm2006geometry}[Proposition 2.10], the Wasserstein space $W^2(\mathbb R^n)$ has non-negative curvature in the sense of Alexandrov. Precisely, 
\begin{lemma} \citet{sturm2006geometry}[Proposition 2.10]
Let $n \geq 1$. The Wasserstein space $W^2(\mathbb R^n)$ equipped with the $L^2$ Wasserstein distance $W_2(\cdot, \cdot)$ has non-negative curvature in the sense of Alexandrov.
\end{lemma}

\begin{proof}[Proof of Theorem 2]
Let $X=W^2(\mathbb R^n)$ and $|\cdot\,,\cdot|$ be the $L^2$ Wasserstein distance. For any $x,y,z \in X$, we denote by $\gamma_{xy}$ ($\gamma_{zx}$) the geodesic connecting $x$ and $y$ (resp. $z$ and $x$). By the above Lemma, $X$ has non-negative curvature in the sense of Alexandrov, hence according to Definition \ref{defn-alex}, one can define the angle between $\gamma_{xy}$ and $\gamma_{zx}$ at $x$, denoted by $\theta$, and we have
$$\theta \geq {\widetilde\angle} yxz :=\arccos \left(  \frac{|x,y|^2 + |z,x|^2 - |y,z|^2}{2 \cdot |x,y| \cdot |z,x|} \right), $$
which implies 
$$
\cos \theta \leq \frac{|x,y|^2 + |z,x|^2 - |y,z|^2}{2 \cdot |x,y| \cdot |z,x|}.
$$
Equivalently,
$$
|y,z|^2 \leq |x,y|^2 + |z,x|^2 -2  |x,y| \cdot |z,x| \cdot \cos \theta.
$$
Hence, we complete the proof.
\end{proof}



\begin{proof}[Proof of Theorem 3]
We derive from the definition of $\ell_{WD}$  and the triangle inequality for the $L^2$ Wasserstein distance that for any $\Theta, \Theta'$,
\begin{equation*}
\begin{aligned}
\| \ell_{WD}(\mathbf{u}_\Theta,\mathbf{v}) - \ell_{WD}(\mathbf{u}_{\Theta'},\mathbf{v}) \| 
& \leq \ell_{WD}(\mathbf{u}_{\Theta'}, \mathbf{u}_{\Theta})  \\ 
&= \ell^{1/2}_{WMD}(\mathbf{u}_{\Theta'}, \mathbf{u}_{\Theta})  \\
&\leq  \left( \sum_{i,j}T_{i,j} d_{i,j} \right)^{1/2}
\end{aligned}
\end{equation*}
where $T_{i,j}$ satisfies
\begin{equation*}
\begin{aligned}
\sum_{j} T_{i,j} &= p_{u_\Theta,i} \quad \forall i, \quad \sum_{i} T_{i,j} &= p_{u_{\Theta'},j} \quad \forall j .
\end{aligned}
\end{equation*}
Take $T_{i,j}=\delta_{ij} \cdot p_{u_\Theta,i}$. According to the assumption that  $\mathbf u_\Theta$ is Lipschitz continuous with respect to the parameters $\Theta$, we have
$$
d_{i,i} = \| u_{\Theta, i} - u_{\Theta', i}\|^2 \leq L \cdot \|\Theta' - \Theta\|^2
$$
for some constant $L>0$.
Hence, we get that
\begin{equation*}
\begin{aligned}
\left( \sum_{i,j}T_{i,j} d_{i,j} \right)^{1/2} &\leq \left( \sum_{i}T_{i, i} \cdot L \cdot \|\Theta' - \Theta\|^2 \right)^{1/2} \\
& = \left( \sum_{i}T_{i, i} \right)^{1/2} \cdot L^{1/2}\cdot  \|\Theta' - \Theta\| \\
& = L^{1/2}\cdot  \|\Theta' - \Theta\| .
\end{aligned}
\end{equation*}
Finally, we got 
\begin{equation*}
\| \ell_{WD}(\mathbf{u}_\Theta,\mathbf{v}) - \ell_{WD}(\mathbf{u}_{\Theta'},\mathbf{v}) \| \leq L^{1/2}\cdot  \|\Theta' - \Theta\| .
\end{equation*}
Hence, we complete the proof.
\end{proof}

\section{Experimental Results of Transformer}
We also evaluate our method using the Transformer architecture on two translation tasks. To prevent the model from over-fitting, we use a Transformer model with a 2-layer encoder and a 2-layer decoder. Other hyper-parameters are almost the same as in \citet{vaswani2017attention}, except for the optimizer. In our experiment, we use SGD to train the model, instead of Adam~\cite{vaswani2017attention}, since our approach is derived under SGD. Results are shown in Table \ref{tab:transformer}, which are consistent with the observations from the LSTM model.
We hope that our approach and theoretical analysis can be extended to the Adam framework~\cite{kingma2014adam, chen2018convergence, reddi2019convergence} in the future.


\begin{table*}[!htbp]
\setlength{\abovecaptionskip}{0.2cm}
\setlength{\belowcaptionskip}{-0.55cm}
\begin{center}
\resizebox{0.45 \columnwidth}{!}{
\begin{tabular}{lcccc}\toprule
 \hline
\textbf{Model} &  \textbf{De$\Rightarrow$En} & \textbf{En$\Rightarrow$De} & \textbf{Vi$\Rightarrow$En} & \textbf{En$\Rightarrow$Vi}  \\\hline
\textsc{CE}     & 29.18   & 24.36 & 25.04   & 26.02
 \\
\textsc{CE+MLM}   &  29.20  & 24.40 & \underline{25.68}  & 25.97 \\
\textsc{CE+BT}     & \textbf{30.01}  & \textbf{25.45} & \textbf{25.77}  & \textbf{27.62} \\
\textsc{CE+OURS}    & \underline{29.25}   & \underline{24.62}  & 25.49  & \underline{26.84} \\\hline
\textsc{WD}  &   28.60  & 24.38 & 24.79 &  \underline{26.43} 
  \\
\textsc{WD+MLM}  & \underline{29.02}  & 24.49 & \underline{25.08} &  26.13  \\
\textsc{WD+BT}   & 28.92   & \underline{24.82}   & 24.88  &  26.38\\
\textsc{WD+Ours}  & \textbf{29.51}   & \textbf{24.96} & \textbf{25.11}  & \textbf{26.66} 
\\\hline
\bottomrule
\end{tabular}
}
\end{center}
\caption{\label{tab:transformer} BLEU scores on two translation datasets using the Transformer model.  CE: Cross-Entropy loss; WD: $L^2$ Wasserstein distance. The best results are in \textbf{bold}, and the second-best results are in \underline{underline}.}
\end{table*}

\end{document}